\documentclass[runningheads]{llncs}
\usepackage[english]{babel}

\usepackage{graphicx}
\usepackage{todonotes}
\usepackage{tabularx}
\usepackage{makecell}
\usepackage{booktabs}
\usepackage{array}

\usepackage{amsfonts}
\usepackage{amsmath}
\usepackage{hyperref}
\usepackage{mathabx}
\usepackage{comment}
\usepackage{enumitem}
\usepackage{wrapfig}

\newcommand{\UEI}{{\mathbb{U}_{ei}}} 
\newcommand{\UACT}{{\mathbb{U}_{act}}} 
\newcommand{\UTIME}{{\mathbb{U}_{time}}} 
\newcommand{\UOT}{{\mathbb{U}_{ot}}} 
\newcommand{\UOI}{{\mathbb{U}_{oi}}} 
\newcommand{\UOMAP}{{\mathbb{U}_{omap}}} 
\newcommand{\UE}{{\mathbb{U}_{event}}} 
\newcommand{\EU}{e = (ei, act, time, omap, vmap) \in \UE}

\newcommand{\PI}{\pi_{ei}}
\newcommand{\PA}{\pi_{act}}
\newcommand{\PT}{\pi_{time}}
\newcommand{\PO}{\pi_{omap}}
\newcommand{\PV}{\pi_{vmap}}
\newcommand{\PN}{N = (P, T, F, l)}
\newcommand{\ON}{ON = (N, pt, F_{var})}
\newcommand{\MI}{M_{init}}
\newcommand{\MF}{M_{final}}
\newcommand{\Markings}{ \mathcal{Q_{ON}}}
\newcommand{\AON}{ AN = (ON, \MI, \MF)} 
\newcommand{\OR}{\preceq_E}
\newcommand{\UL}{\mathbb{U}_{OCEL}}
\newcommand{\UPT}{\mathbb{U}_{pt}}
\newcommand{\UAN}{\mathbb{U}_{AN}}
\newcommand{\ULPHI}{{\mathbb{U}_{OCEL}}\!\!\upharpoonright_{\varphi}}

\newcommand{\pat}{\varphi_{oiwlsp}}

\newcommand{\UP}{\!\!\upharpoonright}

\newcommand{\janik}[1]{\textcolor{black}{#1}}
\newcommand{\technical}[1]{\textcolor{black}{#1}}
\newcommand{\final}[1]{\textcolor{black}{#1}}

\addto\extrasenglish{%
}
\addto\extrasenglish{%
}
\addto\extrasenglish{%
}

\begin{document}

\title{Preventing Object-centric Discovery of Unsound Process Models for Object Interactions with Loops \janik{in Collaborative Systems}: \\Extended Version}

\author{Janik-Vasily Benzin\inst{1}
\and Gyunam Park\inst{2}
%
\and Stefanie Rinderle-Ma\inst{1}
}
\authorrunning{Janik-Vasily Benzin et al.}
%
\institute{
Technical University of Munich, Germany; TUM School of Computation, Information and Technology\\
\email{\{janik.benzin,stefanie.rinderle-ma\}@tum.de}\\
\and Chair of Process and Data Science, Computer Science, RWTH-Aachen University, \email{gnpark@pads.rwth-aachen.de}}
\maketitle              
\begin{abstract}
 Object-centric process discovery (OCPD) constitutes a para\-digm shift in process mining. Instead of assuming a single case notion present in the event log, OCPD can handle events without a single case notion, but that are instead related to a collection of objects each having a certain type. The object types constitute multiple, interacting case notions. The output of OCPD is an object-centric Petri net, i.e. a Petri net with object-typed places, that represents the parallel execution of multiple execution flows corresponding to object types. Similar to classical process discovery, where we aim for behaviorally sound process models as a result, in OCPD, we aim for soundness of the resulting object-centric Petri nets. However, the existing OCPD approach can result in violations of soundness. As we will show, one violation arises for multiple interacting object types with loops that arise in collaborative systems. This paper proposes an extended OCPD approach and proves that it does not suffer from this violation of soundness of the resulting object-centric Petri nets. We also show how we prevent the OCPD approach from introducing spurious interactions in the discovered object-centric Petri net. The proposed extension is prototypically implemented.

\keywords{Process mining \and Object-centric Process Discovery \and Object-centric Petri Nets \and Behavorial Soundness.}
\end{abstract}

\section{Introduction}
\label{sec:intro}

Object-centric process discovery (OCPD) 
shifts the focus of process discovery from classical event logs with a single case notion relating events of a single business process execution to the often more realistic \emph{object-centric event logs} \cite{van_der_aalst_discovering_2020}. Events in an object-centric event log are not related to a single case notion, but instead are related to multiple objects of a certain object type. Each of these objects' execution flows is recorded in the respective events. If multiple objects interact, the corresponding event is related to these multiple objects. Hence, an object-centric event log records the execution history of multiple, interacting execution flows. Since classical process discovery techniques such as the Inductive Miner \cite{leemans_discovering_2013} assume a single case notion, none of the existing techniques can be directly applied to an object-centric event log. Although \emph{flattening} presents an approach to extract single case notions from object-centric event logs, the resulting simple event logs can suffer from the issues of \emph{convergence}, i.e., events have to be duplicated, \emph{divergence}, i.e., the actual order of events is lost, and \emph{deficiency}, i.e., events are missing in the simple event log \cite{van_der_aalst_discovering_2020}. 

To discover holistic process models that highlight the relationship between the various execution flows based on object-centric event logs, \cite{van_der_aalst_discovering_2020} proposes the only existing OCPD approach that discovers holistic process models in the form of \emph{object-centric Petri nets}. Object-centric Petri nets are better suited than \emph{artifact-centric process models} \cite{van_eck_multi-instance_2019,van_eck_guided_2017} and \emph{Object-centric Behavioral Constraint (OCBC) models} \cite{li_automatic_2017,artale_enriching_2019} as starting point for analysing processes spanning multiple object types as artifact-centric process models do not visualize the overall business process in a single diagram and OCBC models tend to quickly become too complex and the corresponding discovery and conformance checking approach are not very scalable \cite{van_der_aalst_discovering_2020}.

Similar to classical process discovery \cite{van_der_aalst_soundness_2011}, OCPD should discover \emph{sound} object-centric Petri nets, since otherwise objects may be left at some point in their execution flow leading to an object-centric Petri net that only accepts the empty language or some activities can never be executed. 
 For the existing OCPD approach, two limitations can be identified for settings in which the OCPD approach fails to discover sound process models or introduces restrictions for the business process that are not supported by the given object-centric event log. These settings can arise in \emph{multi-agent systems} \cite{nesterov_discovering_2023}, \emph{service compositions}, \emph{service orchestrations} \cite{DBLP:journals/jwsr/Rinderle-MaRJ11}, and \emph{process choreographies} \cite{fdhila_verifying_2022}, i.e., \janik{in \emph{collaborative systems} characterized by collaboration between various entities whose workflows are modeled as a business process \cite{sundaramurthy_control_2003,jung_business_2004}}. 

 \janik{
 For collaborative systems, OCPD can discover process models
 by conceptualizing \emph{similarly behaving} agents, \emph{similarly behaving} services, and business processes in process choreographies as object types respectively. As a consequence, object interactions in the object-centric Petri net model synchronous \emph{interaction patterns} between collaborating agents, services and business processes. These interaction patterns can quickly become complex \cite{jung_business_2004} and atypical for classic object-centric settings due to the lack of a central controlling authority \cite{sundaramurthy_control_2003,DBLP:journals/jwsr/Rinderle-MaRJ11}. To handle synchronous interaction patterns in OCPD, \cite{nesterov_discovering_2023} propose to specify a set of interaction pattern models as additional input that result in sound process models by design. Hence, the approach in \cite{nesterov_discovering_2023} depends on specified interaction pattern models as additional input to discover sound process models. As the approach in \cite{nesterov_discovering_2023} does not allow loops in the workflow of an agent, the set of discoverable process models is limited. In contrast, our proposed extensions to overcome the two limitations of OCPD do not require models of interface patterns and can handle loops.}

\begin{table}
    \centering
    \caption{Two fragments $L_1, L_2$ of event logs (separated by the horizontal line in the table). Each event can refer to objects of a certain object type (columns retail credit to service provider are object types). An event is represented by a row (except the header).}
    \scalebox{0.81}{
    \begin{tabular}{lccccccc} 
    \toprule
    id & activity    & timestamp           & retail credit                           &  \makecell{corporate\\credit}                         & coordinator & \makecell{service\\provider} & customer    \\ 
    \midrule
    0ab63  & initialize & 2023-03-10T15:55:28 & $\emptyset$  &       $\emptyset$          &  \{151a3\}       &  $\emptyset$   & $\emptyset$    \\
    6b0b9 & receive request & 2023-03-10T15:55:29 &      $\emptyset$       &    $\emptyset$          &  \{151a3\}         & $\emptyset$  & \{0a3a3\}      \\
    ddf21 & delegate request & 2023-03-10T15:55:30 &      $\emptyset$       &    $\emptyset$          &  \{151a3\}         & \{ec135\}  &  $\emptyset$    \\
    kj875 & fail on request & 2023-03-11T11:00:31 &      $\emptyset$       &    $\emptyset$          &           & \{ec135\}  &   $\emptyset$    \\
    9c7f8 & receive request & 2023-03-11T11:00:32 &      $\emptyset$       &    $\emptyset$          &  \{151a3\}         & \{ec135\}  &   $\emptyset$    \\
    207f2 & escalate request & 2023-03-11T11:00:33 &      $\emptyset$       &    $\emptyset$          &  \{151a3\}         & $\emptyset$  & $\emptyset$      \\
    \midrule
        b2589  & check statement & 2023-03-12T15:50:25 & \{a0287\}      & $\emptyset$                     &    $\emptyset$       & $\emptyset$     & $\emptyset$    \\
    9e602  & check statement  & 2023-03-12T15:50:26 & $\emptyset$                     & \{677f7\} &     $\emptyset$      & $\emptyset$  & $\emptyset$  \\
    65145  & report to authority   & 2023-03-12T15:50:37 & \{a0287\}                      & \{677f7\}                     &    $\emptyset$      & $\emptyset$ & $\emptyset$    \\
    \bottomrule
    \end{tabular}}
    \label{tab:motivating-events-table}
\end{table}

To illustrate the benefits and limitations of applying OCPD to collaborative systems, \autoref{tab:motivating-events-table} contains two object-centric event log fragments $L_1$ and $L_2$. The first fragment is recorded in information systems  that support a ``coordinator'' agent in running a marketplace that matches requests by ``customer agents'' with services to fulfil the requests offered by ``service provider'' agents. To discover a process model for the multi-agent system in terms of the respective agent's workflow and the interaction patterns between agents, we conceptualize the three agent types as object types. 

\janik{The multi-agent system of matching requests results in the following interaction pattern.} After initializing, the ``coordinator'' receives a request from the ``customer'', i.e., event with id ``6b0b9'' records objects of types ``coordinator'' and ``customer''. The ``coordinator'' delegates the request to a matching ``service provider'' that subsequently fails on the request and, thus, gives the request back to the ``coordinator''. From the point of the "coordinator", another request is received (cf. event with id ``9c7f9''). \janik{This request is at last escalated to signal employees of the ``coordinator'' agent that a manual matching has to take place.} From the viewpoint of the ``coordinator'', the activity ``receive request'' represents the DO-part of a loop with the REDO-part being the ``delegate request''. From the viewpoint of ``service provider'', the two activities ``delegate request'' and ``receive request'' are in sequential order due to the first ``receive request'' being unrelated to ``service provider''. The mismatch of activity labels recorded in the event log with the semantics of the real-world activities, namely that the activity receiving a request is semantically dependent on further attributes, e.g., customer agents vs. service provider agents, and the context, e.g., the fact that the second request was already delegated before, \janik{presents a serious problem for the existing OCPD techniques in \cite{van_der_aalst_discovering_2020,nesterov_discovering_2023}. The loop for the ``coordinator'' agent excludes the technique in \cite{nesterov_discovering_2023}. While the technique in \cite{van_der_aalst_discovering_2020} is generic enough that it can be applied, it discovers an object-centric Petri net depicted in \autoref{fig:loop} that deadlocks after transition ``t1'' fired, i.e., the model is unsound. As we assume an event log as the only input and require handling of loops, we extend the OCPD approach to overcome the limitation of discovering unsound process models in light of \emph{object interactions with loops}.}

\begin{figure}
  \centering
  \includegraphics[width=0.86\linewidth]{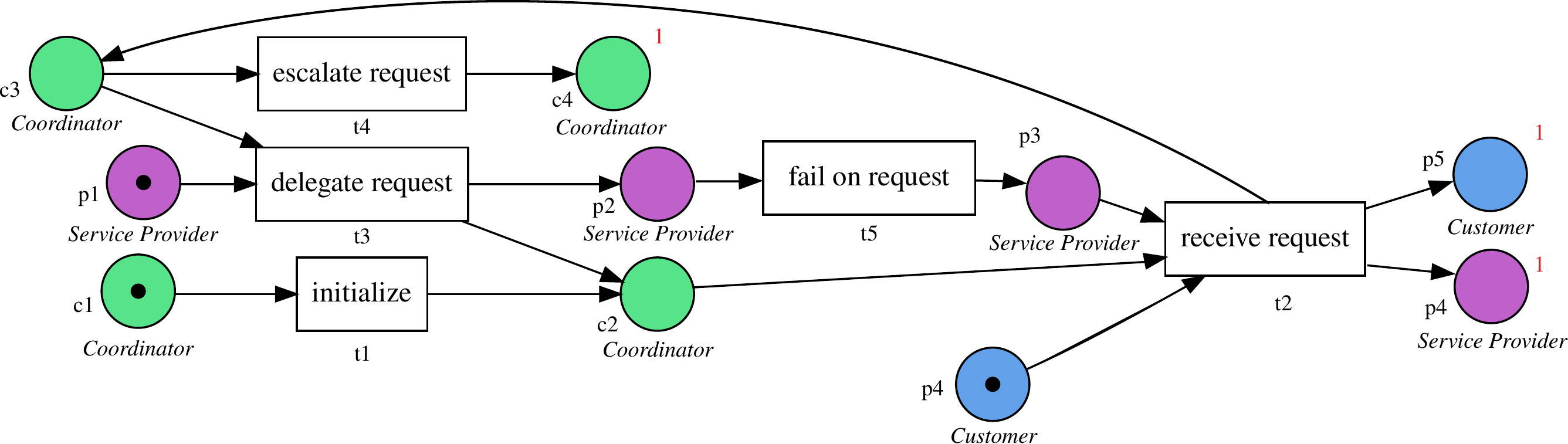}
  \caption{\emph{Unsound} accepting object-centric Petri net $AN_1$ discovered for the first event log fragment $L_1$ by the OCPD approach \cite{van_der_aalst_discovering_2020} due to object interactions with loops contained in $L_1$. Object types of places, e.g. coordinator, are denoted below a place and depicted by color and final markings are denoted as red number next to a place. \janik{Initial markings are depicted by tokens and chosen in \cite{van_der_aalst_discovering_2020} such that the respective agent's workflow starts with its' first activity.}}
  \label{fig:loop}
\end{figure}

\begin{figure}[ht!]
  \centering
  \includegraphics[width=0.74\linewidth]{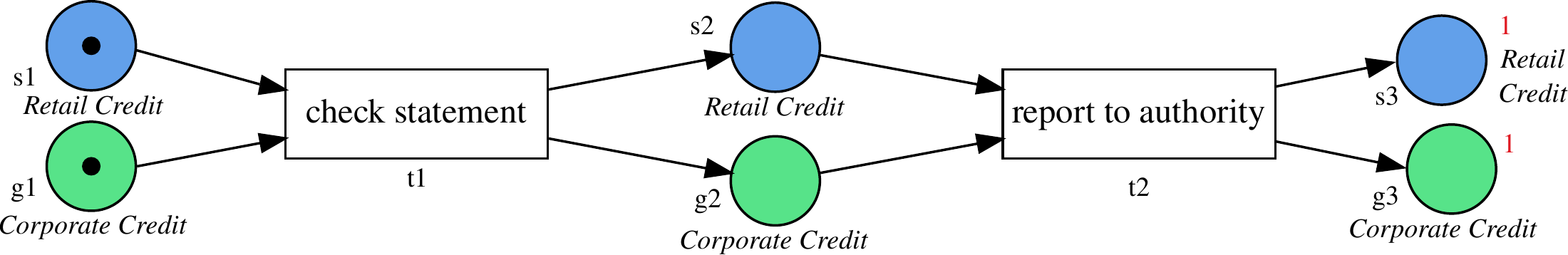}
  \caption{\emph{Sound} accepting object-centric Petri net $AN_2$ discovered for the second event log fragment $L_2$ by the OCPD approach \cite{van_der_aalst_discovering_2020}. Due to the spurious interaction introduced by the OCPD approach, the object-centric Petri net $AN_2$ cannot replay the event log fragment $L_2$, as transition t1 can only fire once.}
  \label{fig:spurios}
\end{figure}

The second fragment $L_2$ shows three events recorded in information systems of a bank. \janik{These information systems support a bank's different business processes ``retail credit'' and ``corporate credit'' transactions \cite{jacobson_credit_2005}. Despite the difference in business processes, every transaction has to be similarly reported in the annual report to an authority \footnote{https://www.ifrs.org/}.} Consequently, the first two events recording the activity of checking credit statements are only related to either ``retail credit'' or ``corporate credit'', but an event recording the activity of reporting to an authority is related to both ``retail credit'' and ``corporate credit''. Although the event log does not record any interaction between the ``retail credit'' and ``corporate credit'' business processes for ``check statement'', the OCPD approach introduces a \emph{spurious interaction} (cf. \autoref{fig:spurios}). The spurious interaction \janik{restricts the execution of the respective business processes in the process orchestration without support by the object-centric event log.}

To discover process models for collaborative systems, we formalize the two identified limitations for the OCPD approach. \final{For the object interactions with loops limitation, we show that the OCPD approach discovers unsound process models}. We propose three extensions of the generic OCPD approach to overcome the limitations. To that end, generalizations of workflow nets and soundness to the object-centric setting are defined.

The remainder is structured as follows. Section \ref{sec:prelim}
introduces OCPD preliminaries.
\autoref{sec:lim} formalizes desired properties of OCPD concepts and \janik{interaction \emph{patterns}} contained in event logs that represent problems for the OCPD approach and identifies two such patterns, i.e., two limitations of the OCPD approach.  \autoref{sec:main} presents three extensions of the OCPD approach to overcome the limitations.
Section \ref{sec:rel} describes related work. Finally, \autoref{sec:concl} concludes this work. 

\section{Preliminaries}
\label{sec:prelim} 
We state basic notations and definitions (\autoref{ssec:basic}) required for OCPD (\autoref{ssec:log}). 
The existing generic OCPD approach \cite{van_der_aalst_discovering_2020} is presented in \autoref{ssec:log}.
d.

\subsection{Basic Notations and Definitions}
\label{ssec:basic}

Given a function $f \in X \rightarrow Y$, we denote it's domain $X$ as $dom(f) = X$ and its range as $ran(f) = \{ y \in Y | \exists_{x \in X} f(x) = y \} \subseteq Y$. Given set $X' \subseteq X$, we denote the restriction of function $f$ on $X'$ as $f\UP_{X'} = \{ (x', f(x')) | x' \in X' \}$\footnote{The restriction is similarly defined and denoted for relation $R \subseteq X \times X$.}. We extend function application to sets $f(X') = \{ y \in Y | \exists_{x \in X'} f(x) = y\}$ for $X' \subseteq X$, also for n-ary functions $f_n \in X_1 \times \ldots \times X_n \rightarrow Y$, $f(X') = f(X' \times \ldots \times X') = \{ y \in Y | \exists_{x_1, \ldots, x_n \in X'} f(x_1, \ldots, x_n) = y \}$. A partial function $f \in X \not\rightarrow Y $ is only defined for elements $x \in X$ that are in the domain $x \in dom(f)$, i.e., $f(x')$ is undefined for $x' \not\in dom(f)$. 

A \emph{trace} $\sigma$ of length $n$ over $X$ is an ordered collection $\sigma \in \{1, \ldots, n\} \rightarrow X$ with $|\sigma| = n$. The set of all sequences over $X$ is denoted as $X^*$. Given a universe of activity names $\UACT$, a \emph{simple event log} is a multiset of traces $L_{SEL} \in \mathcal{B}(\UACT^*) = \mathbb{U}_{SEL}$. Traces record business process executions that are modelled as a \emph{labeled Petri net}. In the following, we use common definitions, semantics and notation for labeled Petri nets, \emph{accepting Petri nets}, \emph{workflow nets} and \emph{soundness} of workflow nets and refer to \cite{van_der_aalst_discovering_2020,van_zelst_translating_2020,van_der_aalst_soundness_2011} for details. 
The universe of accepting Petri nets is denoted as $\mathbb{U}_{APN}$.
A \emph{place-bordered fragment} $N'$ of a labeled Petri net $N$ is a \emph{weakly connected} \emph{subnet} with $N' = (P', T', F\!\!\upharpoonright_{P' \times T'}, l\UP_{T'})$, $P'\subseteq P$, $T' \subseteq T$, $F\UP_{P' \times T'} = F \cap ((P' \times T') \cup (T' \times P'))$ such that all vertices $x' \in P' \cup T'$ that are connected to vertices $x \in (P \cup T) \; \setminus (T \cup P)$ in $N$ that do not belong to $N'$ are places, i.e., $\{ x' | (x', t) \in F \setminus F' \lor (t, x') \in F \setminus F' \} \subseteq P'$. 


A discovery technique $disc$ is a function mapping simple event logs onto accepting Petri nets, i.e., $disc \in \mathbb{U}_{SEL} \rightarrow \mathbb{U}_{APN}$ \cite{van_der_aalst_discovering_2020}. In the following, we denote with $IM \in \mathbb{U}_{SEL} \rightarrow \mathbb{U}_{APN} $ the Inductive process discovery technique (Inductive miner) \cite{leemans_discovering_2013}. The Inductive miner discovers \emph{process trees} that correspond to sound accepting Petri nets \cite{leemans_discovering_2013,van_zelst_translating_2020}, e.g., $\rightarrow (a, \circlearrowleft(b, c)) $ for activities $a, b, c \in \UACT$ a sequence of activity $a$ and a loop with DO-part $b$ and REDO-part $c$.

\subsection{Object-centric Event Logs, Petri Nets and Process Discovery}
\label{ssec:log}

Events in an object-centric event log are defined with the following universes.

\begin{definition}[\janik{Object-centric Event Log} \cite{van_der_aalst_discovering_2020}]
\label{def:ocel}
$L = (E, \OR )$ is an object-centric event log with $E \subseteq \UE $ and $\OR \subseteq E \times E$ such that:
    \begin{itemize}
        \item $\OR$ defines a partial order (reflexive, antisymmetric, and transitive),
        \item $\forall_{e_1, e_2\in E} \PI(e_1) = \PI(e_2) \Rightarrow e_1 = e_2$, 
        \item $\forall_{e_1, e_2\in E} e_1 \OR e_2 \Rightarrow \PT(e_1) \leq \PT(e_2)$,
    \end{itemize}
    given the following universes:
\begin{itemize}
\item $\UE = \UEI \times \UACT \times \UTIME \times \UOMAP \times \mathbb{U}_{vmap}$
is the universe of events. 
\item ${\mathbb{U}_{ei}}$ is the universe of event identifiers,
\item ${\mathbb{U}_{time}}$ is the universe of timestamps, 
\item ${\mathbb{U}_{ot}}$ is the universe of objects types, 
\item ${\mathbb{U}_{oi}}$ is the universe of object identifiers,
\item $type \in \UOI \rightarrow \UOT $ assigns precisely one object type to each object identifier, 
\item $\mathbb{U}_{omap} = \{ omap \in \mathbb{U}_{ot} \not\rightarrow \mathcal{P}(\mathbb{U}_{oi}) | \forall_{ ot \in dom(omap)} \forall_{ oi \in omap(ot)} type(oi) = ot \}$ is the universe of all object mappings indicating which object identifiers are included per type\footnote{We assume that if $ot \not\in dom(omap)$, then $omap(ot) = \emptyset$.}, 
\item ${\mathbb{U}_{att}}$ is the universe of attribute names, 
\item ${\mathbb{U}_{val}}$ is the universe of attribute values, and
\item $\mathbb{U}_{vmap} = \mathbb{U}_{att} \not\rightarrow \mathbb{U}_{val}$ is the universe of value assignments.
\end{itemize}

Given $\EU$, we define the following event projections: $ \PI(e) = ei, \PA(e) = act, \PT(e) = time, \PO(e) = omap, \\\PV(e) = vmap$.
We denote the set of event logs as $\mathbb{U}_{OCEL}$.

\end{definition}

Hence, an event $\EU $ is identified by its unique event identifier $ei$, the activity $act$, a timestamp $time$ and the two mappings $omap$ and $vmap$ that reference the objects related to the event and the attribute values. For the first row of \autoref{tab:motivating-events-table} describing event $e_0$, we have $ \PI(e_0) = \text{0ab63}, \PA(e_0) = \text{initialize}, \PT(e_0) = \text{2023-03-10T15:55:28}, \\\PO(e_0)(\text{coordinator}) = \{\text{151a3}\}, \PO(e_0)(\text{Customer}) = \\\PO(e_0)(\text{service provider}) = \emptyset, \text{and} \;\PV(e_0)(at) = \bot $ for all $at \in \UAN$.

As stated in \autoref{sec:intro}, a single case notion is missing in an event log $L \in \UL$, but instead each event is related to objects of certain object types that is captured in $omap$. Any of these object types $ot$ can be used to flatten the event log into a simple event log $L^{ot}$ having a single case notion defined by objects of that object type. We denote the flattening of an event log as $flatten_{ot}$ and refer to \cite{van_der_aalst_discovering_2020} for a formal definition. 

Note, that after flattening, the flattened event log can be used as a simple event log such that all classical process discovery techniques can be applied on the flattened event log. As described in \autoref{sec:intro}, flattening can introduce serious issues into the flattened, simple event log in the form of divergence, convergence and deficiency. As these issues are not critical to the limitations in \autoref{sec:lim} and extensions in \autoref{sec:main}, we refer to \cite{van_der_aalst_object-centric_2019,van_der_aalst_discovering_2020} for a formalization of these issues. 

By typing places with a function $pt \in P \rightarrow \UOT$, the respective execution flows of objects of a certain type are distinguished in a labeled Petri net. Furthermore, multiple objects of a given object type can be related to a single event $e$ such that variable arcs connected to a transition labeled with the activity of $e$ are possible. Extending labeled Petri nets with place types and variable arcs results in an object-centric Petri net.

    

\begin{definition}[Object-centric Petri Net \cite{van_der_aalst_discovering_2020}] 
\label{def:on}
    An \emph{object-centric Petri net} is a tuple $ON = (N, pt, F_{var})$ where $N = (P, T, F, L)$ is a labeled Petri net, $pt \in P \rightarrow \UOT$ maps places onto object types, and $F_{var} \subseteq F$ is the subset of variable arcs.
\end{definition}

Without the \emph{well-formed} property stated in \cite{van_der_aalst_discovering_2020}, it is possible for a transition to have a variable arc from a place of object type $ot_1$ and a non-variable arc to another place of object type $ot_1$, i.e., objects of type $ot_1$ disappear. As this is not desired, well-formed object centric Petri nets exclude such structures. 

In the following, we omit to say well-formed object-centric Petri net, as any forthcoming object-centric Petri net is meant to be well-formed \janik{and refer to \cite{van_der_aalst_discovering_2020} for a formal definition.} In contrast to the markings of labeled Petri nets, markings in an object-centric Petri net carry object identifiers. Consequently, possible tokens have to mind the respective place type. Due to variable arcs, it is possible for a transition to consume multiple tokens at once during firing. What tokens, i.e. what object identifiers of a certain object type, are consumed is captured in function $b \in \UOMAP $ and denoted in a binding $(t, b)$ for transition $t$. 

\begin{definition}[\janik{Marking, Binding Execution} \cite{van_der_aalst_discovering_2020}] 
\label{def:marking}
    Let $\ON$ be an object-centric Petri net with $\PN$. $Q_{ON} = \{(p, oi) \in P \times \UOI | type(oi) = pt(p)\}$ is the set of possible tokens. A marking $M$ of $ON$ is a multiset of tokens, i.e., $M \in \mathcal{B}(Q_{ON})$. Let $\ON$ be an object-centric Petri net with $\PN$. $B = \{(t, b) \in T \times \UOMAP | dom(b) = tpl(t) \wedge \forall_{ot\in tpl_{nv}(t)} | b(ot) | = 1 \}$ is the set of all possible bindings. $(t, b) \in B$ is a binding and corresponds to the execution of transition $t$ consuming selected objects from the input places and producing the corresponding objects for the output places (both specified by $b$). $cons(t, b) = [(p, oi) \in Q_{ON} | p \in \bullet t \wedge oi \in b(pt(p))]$ is the multiset of tokens to be consumed given binding $(t,b)$. $prod(t,b) = [(p, oi) \in Q_{ON} | p \in t\bullet \wedge oi \in b(pt(p))]$ is the multiset of tokens to be produced given binding $(t, b)$. Binding $(t, b) $ is \emph{enabled} in marking $M \in \mathcal{B}(Q_{ON})$ if $cons(t,b) \leq M$. The occurence of an enabled binding $(t,b)$ in marking $M$ leads to the new marking $M' = M - cons(t,b) + prod(t,b)$. This is denoted as $M \xrightarrow{(t,b)} M'$. 
\end{definition}

\janik{In contrast to labeled Petri nets, the execution of a transition $t \in T$ in an object-centric Petri net consumes objects from its \emph{pre-set} $\bullet t$ and produces objects to its \emph{post-set} $t \bullet$ as tokens are objects.}

Similar to the accepting Petri nets, an accepting object-centric Petri net defines an initial and final marking. 

\begin{definition}[Accepting Object-centric Petri Net  \cite{van_der_aalst_discovering_2020}]
\label{def:an}
    An accepting \\object-centric Petri net is a tuple $AN = (ON, \MI, \MF) $ composed of an object-centric Petri net $\ON$, an initial marking $\MI \in \Markings $, and a final marking $\MF \in \Markings$. The universe of all accepting object-centric Petri nets is denoted as $\mathbf{U}_{AN}$.
\end{definition}

Analogous to a process discovery technique $disc \in \mathbb{U}_{SEL} \rightarrow \mathbb{U}_{APN}$ that discovers accepting Petri nets given a simple event log, OCPD techniques $ocpd \in \UL \rightarrow \mathbb{U}_{AN}$ discover accepting object-centric Petri nets given an event log. 
As the only existing OCPD technique that discovers accepting object-centric Petri nets, the idea behind the generic OCPD approach of \cite{van_der_aalst_discovering_2020} is as follows.

\janik{As depicted in \autoref{fig:ocpd}, the generic OCPD approach $ocpd_{base}$ is decomposed into three general mappings. 
First, $disc^{OT}$ flattens the event log for each of the $n = |OT|$ object types appearing in the log and discovers accepting Petri nets with a process discovery technique $disc$. Second, $merge^n$ merges all discovered accepting Petri nets into a single labeled Petri net by taking the union of places, transitions, the flow relation and labeling function. The merging is defined such that only transitions having the same activity label result in the same transition name, i.e., only transitions with the same activity label result in the same transition in the merged Petri net. Third, $finalize$ adds place types, variable arcs and initial and final markings to yield an accepting object-centric Petri net. All in all, the resulting object-centric Petri net $ON$ is characterized by the $n$ accepting Petri nets discovered for each object type: $ON\UP_{ot} = APN^{ot} = (N^{ot}, T^{ot}, F^{ot}, l\UP_{T^{ot}})$ with $N^{ot} = \{p \in P | pt(p) = ot \}$, $T^{ot} = \{ t \in T | \exists_{p \in \bullet t \cup t \bullet} pt(p) = ot\}$, and $F^{ot} = F \cap ((P^{ot} \times T^{ot}) \cup (T^{ot} \times P^{ot}))$ (cf. \autoref{fig:ocpd}).}

Despite its flexibility with respect to the employed process discovery technique and the method to identify variable arcs, the proposed OCPD approach has two limitations \janik{for discovering process models for collaborative systems.}


\begin{figure}
  \centering
  \includegraphics[width=0.5\linewidth]{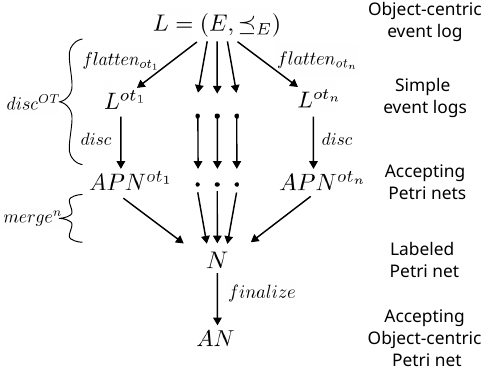}
  \caption{Overview of OCPD approach \cite{van_der_aalst_discovering_2020}.}
  \label{fig:ocpd}
\end{figure}

    

    

\section{Limitations of OCPD Approach: Object Interactions with Loops and Spurios Interactions}
\label{sec:lim}

\janik{First, we elaborate on discovering process models for collaborative systems that lead to the limitations of the OCPD approach in \autoref{ssec:concept}.} To represent and formalize the limitations of the OCPD approach, we propose generalizations of desired properties for object-centric Petri nets, introduce the notion of an \janik{\emph{interaction pattern}} contained in the event log that constitutes a problem for the OCPD approach and isolate the merging and finalizing mappings as critical for discovering desired object-centric Petri nets in a central property of OCPD (\autoref{ssec:struct}). Then, the object interactions with loops limitations is defined as a pattern and shown to be leading to the discovery of unsound object-centric Petri nets by the OCPD approach (\autoref{ssec:limsound}). Finally, the spurious interaction limitation is conceptualized (\autoref{ssec:limreplay}). 

\subsection{\janik{Discovery of Process Models for Collaborative Systems}}
\label{ssec:concept}

The main conceptual idea behind the discovery of process models for collaborative systems, i.e., multi-agent systems, service compositions, service orchestrations and process choreographies, is the interpretation of similarly behaving (= similar workflow) system entities, e.g., agents, services or partner business processes, as object types. We can abstract from the specific entity of the respective collaborative system, e.g., an agent or a service, as long as we have an event log from the collaborative system and aim to discover a process model, because each entity exhibits a workflow recorded through events in the event log. 

If OCPD aims to discover a business process instead of a collaborative system process model, an object type groups objects with a similar workflow, e.g., ``orders'', ``items'', and ``packages'' in \cite{van_der_aalst_object-centric_2019}, such that the OCPD approach discovers WF-nets for each object type. For a business process, object interactions are the result of relationships between object types in the data model of the business process, e.g., a one-to-many relationship between the ``order'' object type and the ``item'' object type. By conceptualizing similarly behaving entities of collaborative systems as object types, the OCPD approach discovers a process model of a collaborative system and the collaboration model of the collaborative system replaces the data model of a business process. 

\begin{figure}
  \centering
  \includegraphics[width=\linewidth]{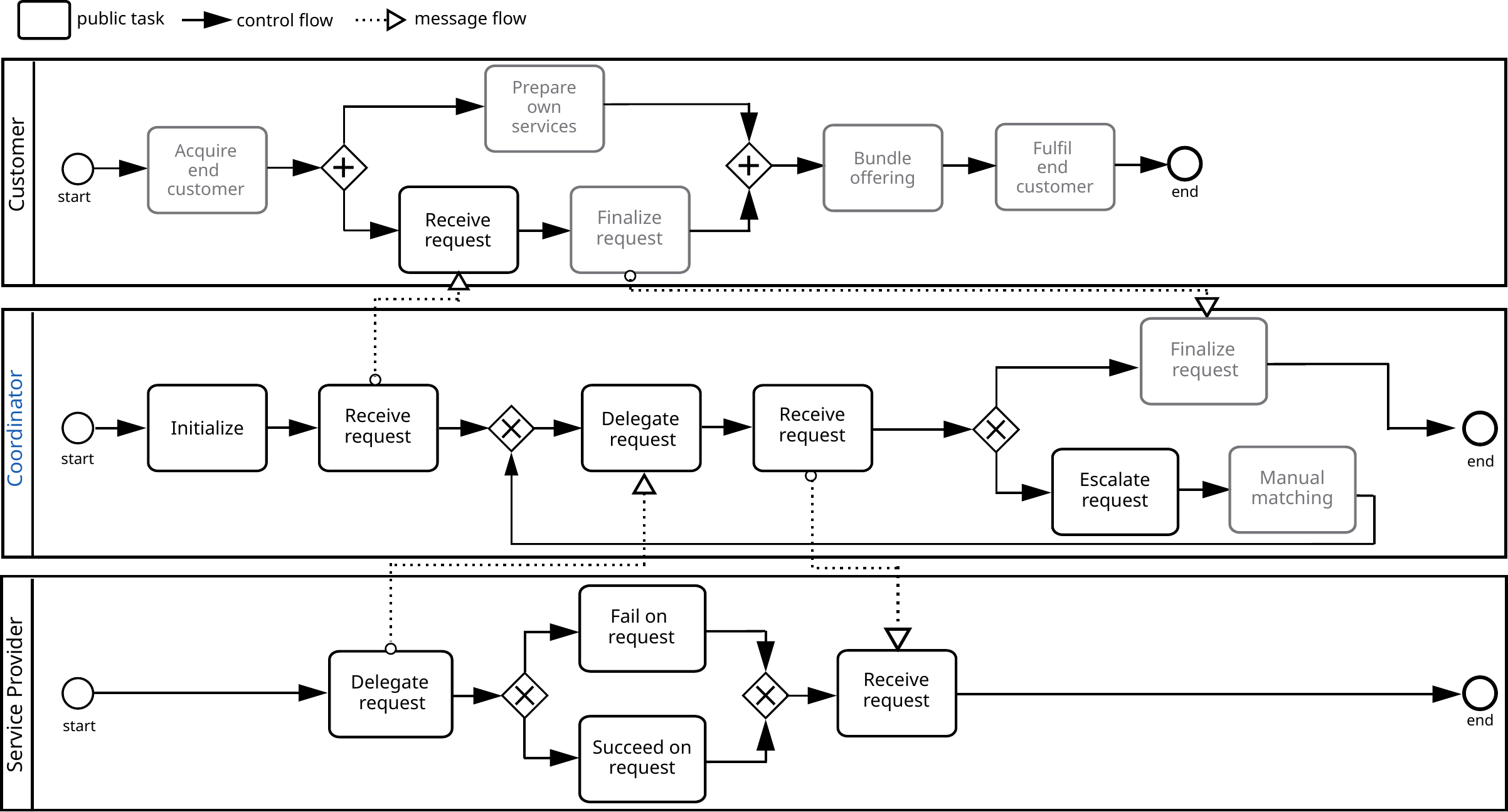}
  \caption{Collaboration model resulting in the object interaction with loops pattern. Activities in gray are not in the event log fragment $L_1$ in \autoref{tab:motivating-events-table}.}
  \label{fig:interaction_da}
\end{figure}

In \autoref{fig:interaction_da}, a collaboration model for the example of a multi-agent system with a ``coordinator'', ``customer'' and ``service provider'' from \autoref{sec:intro} is depicted. Activities of the collaboration model correspond to activity labels of the respective event recorded in the event log fragment $L_1$. Activities in gray are not part of the event log fragment. The collaboration model shows that the ``coordinator'' models receiving requests from the ``customer'' with the same activity label as receiving requests from the ``service provider''. Hence, the interaction pattern depicted in the collaboration model and explained in detail in \autoref{sec:intro} results in the object interactions with loops pattern. The OCPD approach discovers unsound process models (cf. \autoref{fig:loop}) for collaborative systems that exhibit the object interactions with loops pattern, because it has the limitation of expecting similar activity labels to refer to similar object interactions that is particularly problematic for discovering process models of collaborative systems. 

Collaborative systems are characterized by a lack of a central authority governing the system \cite{sundaramurthy_control_2003} such that activity labels have to be interpreted with more caution than in the settings with a central authority, e.g., a business process of a company. A controlled vocabulary cannot be assumed in these situations \cite{DBLP:journals/jwsr/Rinderle-MaRJ11} and interaction patterns in a collaboration model can become more complex than data model relationships \cite{nesterov_discovering_2023}. Moving to the more fine-grained attribute equivalence \cite{DBLP:journals/jwsr/Rinderle-MaRJ11} reveals that the activity label of receiving a request in \autoref{fig:interaction_da} from a ``customer'' and a ``service provider'' is not the same. One of our extensions to the OCPD approach in \autoref{sec:main} builds on the concept of attribute equivalence.

\subsection{Properties of Object-centric Petri Nets}
\label{ssec:struct}

The workflow net as a structural property and its soundness as a behavior property are central concepts in process mining \cite{van_der_aalst_soundness_2011}. In object-centric process mining, these concepts are generalized to \emph{object-centric workflow nets} and \emph{object-centric soundness}. 

\begin{definition}[Object-centric Workflow Net] 
\label{def:ocwfnet}
    An object-centric Petri net $\ON$ is an \emph{object-centric WF-net} iff: 
    \begin{itemize}
        \item Every $ot$-type projection of the well-formed object-centric Petri net is a WF-net, i.e., for every $ot \in ran(pt)$, $ON\UP_{ot} $ is a WF-net. We denote the respective source places as $i_{ot}$ and sink places as $o_{ot}$ of the $ot$-type projection $ON\UP_{ot}$. 
        \item N is weakly connected. 
    \end{itemize}
\end{definition}

Hence, we require each $ot$-type projection to be a WF-net and we only allow object types to occur in an object-centric WF-net for which the event log recorded at least one object interaction with other object types appearing in the event log.

\begin{definition}[Object-centric Soundness] 
\label{def:ocsoundness}
    An accepting object-centric WF-net $\AON$ is \emph{sound} iff:
    \begin{itemize}
        \item Let $ON\!\!\upharpoonright_{ot}$ be the $ot$-type projection, $i_{ot}$ and $o_{ot}$ be its source and sink place. Initial and final marking agree with the source and sink, i.e., $\forall_{ot \in ran(pt)}\\ \;\MI \;\setminus\; (\{ i_{ot} \} \times \UOI) = \emptyset \wedge \MF \;\setminus\; (\{ o_{ot} \} \times \UOI) = \emptyset $.
        \item Option to complete, i.e., $\forall_{M \in R(ON, \MI)} \;\MF \in R(ON, M)$, where \janik{\\$R(ON, M)$ denotes the set of markings reachable from marking $M$.}
        \item No dead transitions, i.e., $\forall_{t \in T} \exists_{M, M' \in R(ON, \MI)}\exists_{b \in \UOMAP} \; M \xrightarrow{(t, b)} M'$.
    \end{itemize}
    We say sound instead of object-centric sound. 
\end{definition}

Since an object-centric WF-net has as many source and sink places as it has object types, the initial and final marking of the accepting object-centric Petri net is only allowed to mark these. The property of "option to complete" and "no dead transitions" is a straightforward generalization for object-centric Petri nets using the binding executions. \janik{In \autoref{fig:sound_model}, a sound object-centric WF-net for the log fragment $L_1$ in \autoref{tab:motivating-events-table} is depicted. Similar to the corresponding collaboration model in \autoref{fig:interaction_da}, the ``coordinator'' first receives a request from a ``customer'' and only later receives a request from the ``service provider''. Despite the object interactions with loops pattern in the log fragment, \autoref{fig:sound_model} shows that a sound object-centric WF-net exists that can model the behavior of the collaborative system.}

\begin{figure}
  \centering
  \includegraphics[width=\linewidth]{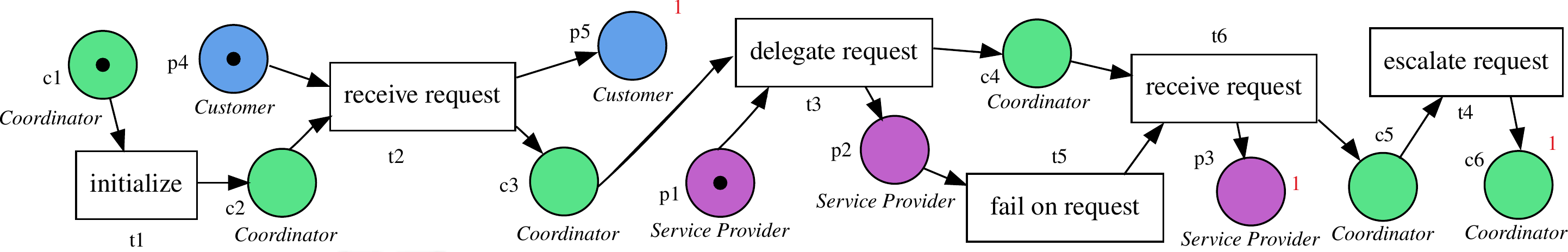}
  \caption{Sound object-centric WF-net for the log fragment $L_1$ that corresponds to the collaboration model of the collaborative system.}
  \label{fig:sound_model}
\end{figure}

\janik{As explained in \autoref{sec:intro}, we only assume an event log as input to the OCPD approach such that the collaboration model in \autoref{fig:interaction_da} cannot be directly used. Nevertheless, the collaboration model defines the interaction patterns between the entities of a collaborative system. As explained in \autoref{ssec:concept}, these interaction patterns correspond to object interactions in the event log. Consequently, we formalize interaction patterns of collaborative systems as object interactions on an event log.}

\technical{The notion of an \emph{object-centric event log pattern} is a logical formula that is satisfied iff the interaction pattern of a collaboration model is contained in the event log. Therefore, it can be checked for a given event log whether the OCPD approach faces the corresponding interaction pattern by evaluating the logical formula defining the pattern given events of the event log.}

\begin{definition}[Object-centric Event Log Pattern] 
\label{def:pattern}
    An event log pattern is a formula $\varphi$ in first-order logic with variables over the universe of events $\UE$ \janik{that corresponds to an interaction pattern of the collaborative system} 
    such that the set of event logs satisfying $\varphi$ is a strict subset of the universe of event logs, i.e., $\ULPHI = \{ L = (E, \OR) \in \UL | \exists_{E_1 \subseteq E} \; \varphi(E_1) \} \subset \UL $. We say that if $L\in \ULPHI$, the event log $L$ contains the pattern $\varphi$. We denote the universe of event log patterns as $\mathbb{U}_{pt}$. \technical{Given $L \in \ULPHI$, we define the event log $L_{\neg \varphi}$ as the event log in which all the events that satisfy pattern $\varphi$ are removed, i.e., $L_{\neg \varphi} = ( E_{\neg \varphi}, \preceq_{E_{\neg \varphi}}) $ with $  E_{\neg \varphi} $ the largest subset of $E$ such that $ \neg \varphi (E_{\neg \varphi}) $.}
\end{definition}

\janik{In \autoref{fig:interaction_da}, the first interaction between the ``customer'' and the ``coordinator'' corresponds to the object-centric event log pattern $\varphi$ that is true for an event $e$ iff there exist two distinct object types $ot_1, ot_2$ appearing in the event log such that the event records objects of both types, i.e. $\PO(e)(ot_1) \neq \emptyset \wedge \PO(e)(ot_2) \neq \emptyset$. By defining interaction patterns on event logs, we do not require collaboration models in discovering process models of collaborative systems with the OCPD approach.} 
\technical{As defined, given a pattern and an event log containing the pattern, we denote the largest event (sub)log by $L_{\neg \varphi}$ that does not contain the pattern anymore. This allows us to isolate the problematic events from the other events in an event log, thus, enabling us to prove for OCPD approach extensions that the extension discovers sound object-centric Petri nets given event logs containing the pattern. }

\begin{definition}[$PT$-Sound Object-centric Process Discovery Technique] 
\label{def:sound}
    \technical{
    Let $PT \subseteq \UPT$ be a set of event log patterns and $L = (E, \OR) \in UL$ be an event log with object types $OT$. Let $ocpd \in \UL \rightarrow \UAN$ be an OCPD approach that can be decomposed into three mappings $ocpd = finalize \; \circ\, merge^n \;\circ\, disc^{OT}$ with $disc^{OT} \in \UL \rightarrow \mathbb{U}_{APN} \times \ldots \times \mathbb{U}_{APN} $, $merge^n \in \mathbb{U}_{APN_1} \times \ldots \times \mathbb{U}_{APN_n} \rightarrow \mathbb{U}_{ON} $, and $finalize \in  \mathbb{U}_{N} \rightarrow \UAN $. OCPD approach $ocpd$ is $PT$-sound iff for every 
    object-centric event log pattern $\varphi \in PT$ and for every event log $L \in \ULPHI$ containing the pattern $\varphi$ such that $ ocpd(L_{\neg \varphi}) $ is a sound object-centric WF-net and $ disc^{OT}(L)$ is a $|OT|$-length tuple of sound WF-nets it holds that:}
    \technical{
    \begin{itemize}
        \item $finalize \;\circ\, merge^n (disc^{OT}(L))$ is a sound object-centric WF-net.
    \end{itemize}}
\end{definition}

\technical{
The notion of a $PT$-sound OCPD technique maintains the flexibility inherent in the proposed generic OCPD approach and isolates potentially problematic constructions of merging and finalizing from potential problems during the process discovery on simple event logs. Furthermore, the patterns in focus, i.e., members of $PT$, are isolated from potential further problematic patterns in an event log by requiring discovery of a sound object-centric WF-net for the sublog $L_{\neg \varphi}$. As mentioned, this isolation enables proving sound OCPD techniques for single or multiple patterns, in particular for the object interactions with loops pattern.}

\subsection{Object Interactions with Loops}
\label{ssec:limsound}

As described in \autoref{sec:intro} for the example in \autoref{tab:motivating-events-table}, the OCPD approach discovers unsound process models in light of the object interactions with loops pattern. The pattern is characterized by two object types in an object-centric event log for which a loop of length one is recorded for the first object type and the activities recorded in the DO-part of the loop do not match with respect to the recorded object interactions. 

\begin{definition}[\janik{Object Interactions with Loops Pattern}] 
\label{def:oiwl}
    Let $L = (E, \OR )$ be an object-centric event log. The object-centric event log pattern \\$\varphi_{oiwlp}(e_1, e_2, e_3, e_4) $ is true iff for events $e_1, e_2, e_3, e_4$ it holds that $\exists_{ot_1, ot_2 \in \UOT}\\ \exists_{act_1, act_2 \in \UACT}$ such that:
    \begin{enumerate}[label=\roman*]
        \item all four events are different, i.e., $\PI(e_1) \neq \PI(e_2) \neq \PI(e_3) \neq \PI(e_4)$,
        \item the activities of the first three events are a loop of length one, i.e., $\PA(e_1) = act_1 \wedge \PA(e_2) = act_2 \wedge \PA(e_3) = act_1 \wedge e_1 \OR e_2 \OR e_3 \wedge \PT(e_1) < \PT(e_2) < \PT(e_3)$,
        \item the first three events share an object of object type $ot_1$, i.e., \\$\PO(e_1)(ot_1) \cap \PO(e_2)(ot_1) \cap \PO(e_3)(ot_1) \neq \emptyset$,
        \item only the second and third event share an object of object type $ot_2$, i.e., \\$\PO(e1)(ot_2) \cap \PO(e_2)(ot_2) \cap \PO(e_3)(ot_2) = \emptyset \wedge \PO(e_2)(ot_2) \cap \PO(e_3)(ot_2) \neq \emptyset$, and 
        \item events with activity $act_1$ before the second event do not share objects with the second event for object type $ot_2$, i.e., $\forall_{oi \in \PO(e_2)(ot_2)} e_4 \OR e_2 \wedge \PA(e_4) = act_1 \rightarrow \PO(e_4)(ot_2) \cap \{ oi \} = \emptyset $.
    \end{enumerate}
    is the object interactions with loops object-centric event log pattern (object interactions with loops pattern).
\end{definition}

\janik{The log fragment $L_1$ in \autoref{tab:motivating-events-table} contains the object interactions with loops pattern $\varphi_{oiwlp}$, as the second row is event $e_1$, the third row is event $e_2$, the fourth row is event $e_4$ and the fifth row is event $e_3$ in the pattern. We conject that both of our extensions proposed in \autoref{sec:main} discover sound process models for collaborative systems exhibiting the object interactions with loops pattern. However, we only prove our conjecture for the following subpattern of the object interactions with loop pattern.}

\begin{definition}[Object Interactions with Loops \janik{Subpattern}] 
\label{def:oiwlp}
\technical{
    Let $L = (E, \OR )$ be an object-centric event log. The object-centric event log pattern \\$\varphi_{oiwlsp}(e_1, e_2, e_3, e_4) $ is true iff for events $e_1, e_2, e_3, e_4$ it holds that \\$\exists_{ot_1, ot_2 \in \UOT} \exists_{act_1, act_2 \in \UACT}$ such that:
    \begin{enumerate}[label=\roman*]
        \item all four events are different, i.e., $\PI(e_1) \neq \PI(e_2) \neq \PI(e_3) \neq \PI(e_4)$,
        \item the activities of the first three events are a loop of length one, i.e., $\PA(e_1) = act_1 \wedge \PA(e_2) = act_2 \wedge \PA(e_3) = act_1 \wedge e_1 \OR e_2 \OR e_3 \wedge \PT(e_1) < \PT(e_2) < \PT(e_3)$,
        \item the first three events share an object of object type $ot_1$, i.e., \\$\PO(e_1)(ot_1) \cap \PO(e_2)(ot_1) \cap \PO(e_3)(ot_1) \neq \emptyset$,
        \item the first three events do not share an object of object type $ot_2$, i.e., \\$\PO(e1)(ot_2) \cap \PO(e_2)(ot_2) \cap \PO(e_3)(ot_2) = \emptyset$,
        \item the second and third event share an object of object type $ot_2$, i.e., \\$\PO(e_2)(ot_2) \cap \PO(e_3)(ot_2) \neq \emptyset$,
        \item the fourth event and no other event contains the activities $act_1$ or $act_2$, i.e., $\PA(e_4) \neq act_1 \wedge \PA(e_4) \neq act_2 \wedge \not \exists_{e_5 \in E} \; \PA(e_5) = act_1 \wedge \PA(e_5) = act_2$, and
        \item before the third event records the second execution of the DO-part of the object execution workflow of object type $ot_1$, arbitrary events can be recorded for both object types as long as they do not share activities with events "outside" of the loop and introduce further object interactions, i.e., $\PI(e_1) \neq \PI(e_2) \neq \PI(e_3) \neq \PI(e_4) \wedge e_1 \OR e_4 \OR e_3 \wedge \PT(e_1) < \PT(e_4) < \PT(e_3) \wedge ((\PO(e_4)(ot_1) \cap \PO(e_1)(ot_1) \neq \emptyset \wedge \PO(e_4)(ot_2) = \emptyset) \lor (\PO(e_4)(ot_2) \cap \PO(e_2)(ot_2) \neq \emptyset \wedge \PO(e_4)(ot_1) = \emptyset)) \\\wedge \not \exists_{e_5 \in E} \; (\PT(e_5) < \PT(e_1) \lor \PT(e_3) < \PT(e_5)) \wedge \PA(e_5) = \PA(e_6) $.
    \end{enumerate}
    is the object interactions with loops object-centric event log pattern (object interactions with loops pattern).}
\end{definition}

The object interactions with loops \janik{subpattern} additionally requires that no further object interactions are recorded "within" the loop of length one and that the activities recorded "within" the loop do not occur outside of the loop. Hence, the activity $act_1$ in the DO-part of the loop represents a clearly defined border of the subpattern. This border allows us to show that the object-centric WF-net discovered for the events of the subpattern is a place-bordered fragment of the object-centric WF-net discovered for the whole event log that contains the subpattern.

\begin{lemma}

\label{lemma:fragment}
\technical{
    Given the object interactions with loops subpattern $\varphi = \varphi_{oiwlsp} \in \mathbb{U}_{pt}$, let $L \in \UL\!\!\upharpoonright_{\varphi}$ be an event log that contains the object interactions with loops subpattern, $OT$ be the object types and $A$ the activities appearing in event log $L$. Let $L_{\varphi} = ( E \;\setminus\; E_{\neg \varphi}, \preceq_{E \;\setminus\; E_{\neg \varphi}}) $ be the event (sub-)log containing all events that satisfy the subpattern $\varphi$ and $OT_{\varphi} \subseteq OT$ be the object types appearing in $L_{\varphi}$. If $ ocpd(L_{\neg \varphi}) = AN_{\neg \varphi} $ is an object-centric WF-net, 
    $ ocpd(L_{\varphi}) = AN_{\varphi} $ is an object-centric WF-net, and the Inductive miner $IM$ is applied as process discovery technique in $disc^{OT}$, then $ AN_{\varphi} $ is a place-bordered fragment of $ocpd(L) = AN$.}
\end{lemma}
\begin{proof}
\technical{
    From the subpattern's definition (vi) and (vii) (cf. \autoref{def:oiwlp}), it follows that the two object-centric WF-nets $ AN_{\neg \varphi} $ and $AN_{\varphi}$ do not share any activities and, thus, any transition labels. From the subpattern's definition (ii) and (vii), every trace $\sigma \in L_{\varphi}^{ot_1}$ starts with $act_1$ and ends with $act_1$, while $act_1$ not occurring more than twice in any of the traces. Hence, the Inductive miner finds a loop cut at first applied on $L_{\varphi}^{ot_1}$ that is also a node in the process tree discovered for $L^{ot_1}$. From the subpattern's definition (ii) and (vii), every trace $\sigma \in L_{\varphi}^{ot_2}$ starts with $act_2$, ends with it, and neither $act_1$ nor $act_2$ occurring a second time in any of the traces. Hence, the Inductive miner finds a sequence cut first applied on $L_{\varphi}^{ot_2}$ that is also a node in the process tree discovered for $L^{ot_2}$. Altogether, by definition of the process tree operators and their transformations to WF-nets \cite{van_zelst_translating_2020}, the object-centric WF-net $ AN_{\varphi} $ is a place-bordered fragment of $AN$. $\square$}
\end{proof}
\technical{
The place-bordered fragment $AN_{\varphi}$ discovered for $L_{\varphi}$ constitutes the fragment of $AN$ that is critical in the following statements about $PT$-soundness (cf. \autoref{thm:unsound}, \autoref{thm:dasound} and \autoref{thm:sasound}).}

\begin{theorem}
\label{thm:unsound}
\technical{For $PT = \{ \varphi_{oiwlsp} \}$ the OCPD approach $ocpd_{base}$ is $PT$-unsound.}
\end{theorem}
\begin{proof}
    We prove the theorem by providing a counterexample for the opposite. Let $L$ be an event log with the six events with event ids 0ab63, 6b0b9, ddf21, kj875, 9c7f8 and 207f2 in \autoref{tab:motivating-events-table}. The four events with event ids 6b0b9, ddf21, kj875 and 9c7f8 satisfy $\varphi_{oiwlsp}$ such that event log $L $ contains subpattern $\varphi_{oiwlsp}$. 
    $L_{\neg \varphi_{oiwlsp}}$ consists of events with event ids 0ab63 and 207f2, for which $ocpd_{base}$ trivially discovers a sound object-centric WF-net. 
    As can be seen in \autoref{fig:loop}, $ocpd_{base}$ discovers sound WF-nets for all three object types coordinator, Customer and service provider, e.g. with $\alpha^+$ miner \cite{de_medeiros_workflow_2003}, Inductive miner \cite{leemans_discovering_2013}, Heuristics miner \cite{weijters_flexible_2011}, ILP miner \cite{van_zelst_discovering_2018} or Region-based miner \cite{van_der_aalst_process_2010} for $disc^{ot}(L^{ot})$. The accepting object-centric Petri net in \autoref{fig:loop} shows that after transition "t1" labeled with activity "initialize" no further transition can fire, thus violating the "no dead transitions" property required for an object-centric sound object-centric WF-net. $\square$
\end{proof}

\autoref{thm:unsound} shows that the existing OCPD approach discovers unsound process models for all event logs containing the object interactions with loops subpatterns. In \autoref{sec:main}, we propose two extensions to the OCPD approach such that the extended approaches discover sound process models for all event logs containing the subpattern. 

\subsection{Spurious Interactions}
\label{ssec:limreplay}

The spurious interactions pattern $\varphi_{si} \in U_{pt}$ is contained in an event log $L \in \UL\!\!\upharpoonright_{\varphi_{si}}$ iff for two different object types $ot_1, ot_2$ appearing in the event log there exist at least two different events $e_1, e_2$ with the same activity label $\PA(e_1) = \PA(e_2) = act$ such that these two events do not share objects of type $ot_1, ot_2$, one event of the two events is related to an object of type $ot_1$, the other related to an object of type $ot_2$ and there does not exist any other events $e_3$ with the same activity label that share objects of types $ot_1, ot_2$. 

Since the OCPD approach $ocpd_{base}$ expects in the merging of accepting Petri nets $merge^n$ for transitions $t_1 \in disc(L^{ot_1}), t_2 \in disc(L^{ot_2})$ with the same activity label $l^{ot_1}(t_1) = l^{ot_2}(t_2) = act$ that these are supported by object interactions in the event log, all transitions with the same activity label are merged into a single transition of the merged labeled Petri net (cf. \autoref{fig:spurios}). In case of an event log containing the spurious interactions pattern, this merging results in merged transitions for activity label $act$ that are not supported by the event log. Consequently, the accepting object-centric Petri net restricts the behavior of the process model, i.e., the respective object execution workflows for $ot_1, ot_2$ cannot execute the transition labeled with $act$ independently, but have to synchronize. However, this restriction on the behavior in the process model is not supported by the event log. In \autoref{ssec:spurious}, the approach to overcome this limitation is presented.

\section{Approaches to Overcome Limitations: Object Interactions with Loops and Spurious Interactions}
\label{sec:main}

The two limitations object interactions with loops and spurious interactions are both caused by the expectation of the merging $merge^n$ of the OCPD approach $ocpd_{base}$ that similar activity labels indicate similar object interactions. Despite the same cause in the OCPD approach, the result of the two limitations on the discovered process model are different, as the former causes the $ocpd_{base}$ to discover unsound process models, while the latter does not affect soundness, but restricts the behavior possible in the process model without support by the event log. In \autoref{ssec:patterns}, two extensions to $ocpd_{base}$ are proposed that overcome the object interactions with loops pattern and it is shown for its subpattern that these extensions discover sound process model despite the event log containing the problematic subpattern. In \autoref{ssec:spurious}, an extension to $ocpd_{base}$ is proposed that overcomes the spurious interactions limitation by removing the restriction introduced to the process model without support of the event log.

\subsection{Object Interactions with Loops Pattern}
\label{ssec:patterns}

We propose two different approaches, \emph{different activity} and \emph{similar activity}, of extending the OCPD approach $ocpd$ to overcome the limitation of discovering unsound process models in light of event logs containing the object interactions with loops pattern. In the following, we present the two approaches as extensions to the $merge^n$ and $finalize$ mappings of the OCPD approach (cf. \autoref{ssec:log}).

\subsubsection*{Different Activity Extension of OCPD Approach}

\janik{Given an event log $L = (E, \OR) $ containing the object interactions with loops pattern. Identify all events $e \in E$ that constitute the object interactions with loops pattern and identify event $e_1$ of the pattern, i.e., the event with id ``6b0b9'' and activity label ``receive request'' in \autoref{tab:motivating-events-table} that records the first execution of the receiving a request from a ``customer'' (cf. collaboration model \autoref{fig:interaction_da}). Relabel every event $e \in E$ that matches the object interaction with loops pattern in the form of $e_1$ to a new activity label $act \in \UACT \;\setminus\; \{ \PA(e) | e \in E \} $. Then, apply $ocpd_{base}$. Finally, relabel the transition $t \in T$ of the accepting object-centric Petri net that is labeled with the new activity label $act$ back to the original activity label $act_1$.}

\janik{The process model in \autoref{fig:sound_model} is discovered with the OCPD approach extended with the different activity extension. The process model has two transitions with the activity label ``receive request'' corresponding to $act_1$ in the object interactions with loops pattern. The process model is a sound object-centric WF-net, because the problematic object interaction in the loop discovered without the extension (cf. \autoref{fig:loop}) is now separated into two distinct transitions.}

In general, the OCPD approach extended with different activity $ocpd_{da}$ breaks the problematic loop recorded for object type $ot_1$ up by relabeling the first DO-part execution such that process discovery techniques discover a sequential relationship between the relabeled activity and the subsequent activities instead of a loop for $L^{ot_1}$. \technical{For the Inductive miner, we prove this conjecture.}
\begin{lemma}
\label{lemma:im}
    \technical{
    Given the object interactions with loops subpattern $\varphi_{oiwlsp} \in \mathbb{U}_{pt}$, let $L \in \UL\!\!\upharpoonright_{\pat}$ be an event log that contains the object interactions with loops subpattern, $OT$ be the object types and $A$ the activities appearing in event log $L$. Let $L_{\varphi_{oiwlsp}} = ( E \;\setminus\; E_{\neg \pat}, \preceq_{E \;\setminus\; E_{\neg \pat}}) $ be the event (sub-)log containing all events that satisfy the subpattern $\pat$ and $OT_{\pat} \subseteq OT$ be the object types appearing in $L_{\pat}$. Let $ot_1, ot_2 \in OT_{\pat}$ be the two object types that are instantiated for the two variables of the same name in $\pat$ and $act_1, act_2 \in A$ that are instantiated for the two variables of the same name in $\pat$ for satisfying the subpattern. Then, for $flatten_{ot_1}(relabel_{\pat}(L_{\pat})) = L^{ot_1}$ the Inductive miner discovers a process tree with a sequence operator at the root of the tree, i.e. $IM(L^{ot_1}) = \,\rightarrow (act, IM_1(split_\rightarrow(L^{ot_1})), act_1)$ with $act \in \UACT \; \setminus\; A$ the new activity label \footnote{$split_\rightarrow$ is the Inductive miner's split into sublog function for the sequence operator \cite{leemans_discovering_2013,van_der_aalst_process_2016}.}, and for $flatten_{ot_2}(relabel_{\pat}(L_{\pat})) = L^{ot_2}$ the Inductive miner discovers a process tree with a sequence operator at the root of the tree, i.e. $IM(L^{ot_2}) = \rightarrow(act_2, IM_1(split_\rightarrow(L^{ot_2}), act_1) $.}
\end{lemma}
\begin{proof}
\technical{
    Because of (ii), (iii) and (vi) in \autoref{def:oiwlp} and by definition of \\$relabel_{\pat}$, $L^{ot_1}$ only contains traces that start with $act$, end with $act_1$, and do not contain $act_1$ in between, i.e., Inductive miner first applies a sequence cut on the directly-follows graph built for $L^{ot_1}$. From \autoref{def:oiwlp} (iv) and the definition of $relabel_{\pat}$, it follows that no events are relabelled in $L^{ot_2}$. From \autoref{def:oiwlp} (ii), (iv), (v) and (vi), it follows that $L^{ot_2}$ only contains traces that start with $act_2$, end with $act_1$ and do not contain neither $act_1$ nor $act_2$ in between such that the Inductive miner will first apply a sequence cut. $\square$}
\end{proof}

\technical{By considering the now "aligned" sequential relationship of activities (discovered by the Inductive miner) for which the two object interactions between $ot_1$ and $ot_2$ are recorded, the previously problematic place-bordered fragment $AN_{\varphi}$ (cf. \autoref{thm:unsound}) becomes sound such that $ocpd_{da}$ becomes $\{\pat\}$-sound.}

\begin{theorem}
\label{thm:dasound}
\technical{If the Inductive miner is used for process discovery on flattened event logs, then OCPD Approach extended with different activity $ocpd_{da}$ is a $\{\pat\}$-sound OCPD technique.}
\end{theorem}

\begin{proof}
\technical{
Given the object interactions with loops subpattern $\varphi = \pat \in \mathbb{U}_{pt}$, let $L \in \UL\!\!\upharpoonright_{\varphi}$ be an event log containing the subpattern with object types $OT = \{ot_1, \ldots, ot_n \} \subseteq \UOT$ and activities $A \subseteq \UACT$ appearing in the event log such that $ocpd_{da}(L_{\neg \varphi}) = AN_{\neg \varphi}$ is a sound object-centric WF-net (I) and $disc_{da}^{OT}(L)$ is a $|OT|$-length tuple of sound WF-nets (II). Let $L_{\varphi} = ( E \;\setminus\; E_{\neg \varphi}, \preceq_{E \;\setminus\; E_{\neg \varphi}}) $ be the event log containing all events that satisfy the subpattern $\varphi$. Assume that $ocpd_{da}(L) = AN$ is an unsound object-centric WF-net (III). 
From (I), it follows that object-centric soundness can only be violated through discovery on $L_{\varphi}$. From (II), it follows that each of the object types appearing in $L_{\varphi}$ result in sound WF-nets through discovery by $disc_{da}$. Hence, object-centric soundness of $ocpd_{da}(L) = AN$ can only be violated through $finalize \;\circ\, merge^n (disc^{OT}(L))$. From \autoref{lemma:fragment} and \autoref{lemma:im}, it follows that the discovered object-centric Petri net $ocpd_{da}(L_{\varphi}) = AN_{\varphi}$ discovered for $L_{\varphi}$ is a place-bordered fragment of the overall object-centric Petri net $ocpd_{da}(L)$. From (I), (II) and (III), it follows that $AN_{\varphi}$ is the largest place-bordered fragment of $AN$ that is object-centric unsound. From \autoref{lemma:im} and by definition of $finalize_{da}$, the initial and final marking of $AN_{\varphi} = (ON_{\varphi}, M_{init, \varphi}, M_{final, \varphi} )$ agree with the $ot$-type projections' $ON_{\varphi}\!\!\upharpoonright_{ot}$ source and sink for $ot \in OT_{\varphi}$. From (II), \autoref{lemma:im}, and by definition of the subpattern (cf. \autoref{def:oiwlp}) and the mapping $merge_{base}^n$, it follows that the only markings $M \in R(ON_{\varphi},  M_{init, \varphi})$ that do not have the "option to complete" property for object-centric soundness, can be markings in which a transition $t \in T_{\varphi}$ synchronizes the execution flows of two object types, i.e., $ tpl(t) = \{ ot_1 , ot_2 \} \subseteq OT_{\varphi}$ and two object types $ot_1, ot_2 $ that were instantiated to satisfy subpattern $\varphi$, cannot fire anymore (coinciding with the "no dead transition" property). From (iii-vii) of the subpattern's definition, it follows that there are exactly two events $e_2 $ and $e_3$ in the event log $L_{\varphi}$ that constitute an object interaction. Since the only event $e_1$ in $L_{\varphi}$  with the same activity label $act_1$ as $e_3$ is relabeled to an new activity label $act \in \UACT \; \setminus\; A$, the mapping $merge_{base}^n$ merges exactly two transitions that are labeled with $act_1$ and $act_2$, resulting in two synchronizing transitions $t_1, t_2 \in T_{\varphi}$. Hence, for at least one of the two transitions $t_1$ and $t_2$ it must hold, that there does not exist a marking $M \in R(ON_{\varphi}, M_{init, \varphi})$ such that the transition is enabled. From \autoref{lemma:im}, these two labels are sequentially related such that the merging cannot have introduced a marking $M$ into the set of reachable markings $R(ON_{\varphi}, M_{init, \varphi})$ that does not enable either of the two transitions anymore, contradicting the assumption that $AN$ is an object-centric unsound WF-net. $\square$}
\end{proof}

\technical{Hence, the extension construction of the OCPD approach results in a \\$\{\pat\}$-sound OCPD technique.} Consequently, we can apply the extended OCPD approach to event logs recorded from multi-agent systems, service compositions and service orchestrations for cases in which attribute equivalence indicates different real-world activities despite the same activity label. If an object of type $ot_2$ in event $e_1$ of subpattern $\pat$ with an identifier also recorded for the events $e_2$ and $e_3$ is missing or given domain knowledge the activities of the two events $e_1$ and $e_3$ refer to the same real-world activity, then the different activity extension should not be used. 

The reasoning for the similar activity interpretation, missing objects or domain knowledge, must be further differentiated, since the missing object can be added to the event $e_1$ after careful analysis of the event log such that the original OCPD approach $ocpd_{base}$ can be used. If there is no missing object and the activity labels recorded in $e_1$ and $e_3$ refer to the same real-world activity, then the similar activity extension should be used. 

\begin{figure}
  \centering
  \includegraphics[width=0.75\linewidth]{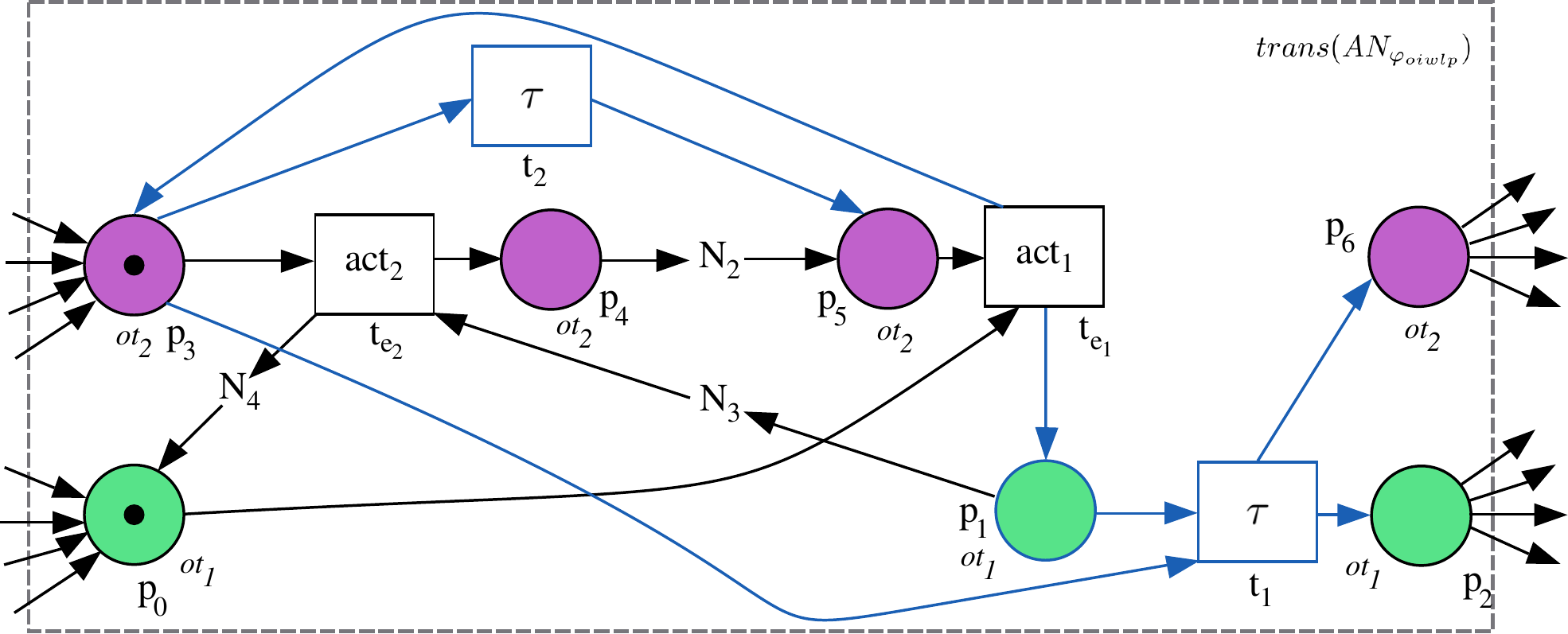}
  \caption{\emph{Sound} accepting object-centric Petri net fragment $trans(AN_{\pat})$ discovered by the OCPD approach extended with similar activity extension by first applying the OCPD approach $ocpd_{base}$ and then transforming the place-bordered fragment $AN_{\pat}$ such that it is sound. The newly added place $p_1$, the two new silent transitions $t_1, t_2$ and the respective new arcs are highlighted in blue. $N_1, N_2, N_3$ are the place-bordered fragments that are discovered for events $e_4$ in $\pat$.}
  \label{fig:sa}
\end{figure}

The idea for the construction of the extension is \janik{to first applying the original approach $ocpd_{base}$. Then, the extension transforms the unsound object-centric WF-net such that the workflow of object type $ot_1$ becomes a loop with DO-part $act_1$ and REDO-part $act_2$, i.e., the workflow of object type $ot_2$ can "mimick" the loop of object type $ot_1$. \autoref{fig:sa} depicts a fragment of the object-centric WF-net that is discovered on all events that satisfy the object interactions with loops pattern only.} The fragment $AN_{\pat}$ is entered by marking $[p_0, p_3]$\footnote{We abstract from the object identifier of the marking in an object-centric Petri net.}. Without the new silent transition $t_2$, the transition $t_{e_1}$ labeled with $act_1$ is never enabled. The only enabled transition is $t_2$ such that after firing $t_2$, transition $t_{e_1}$ is enabled. Hence, the first executed activity is $act_1$. After fragment $N_2$ is executed (corresponding to the activities of events $e_4$ in $\pat$), the object types $ot_1, ot_2$ can synchronize on $t_{e_2}$ labeled with $act_2$. Then, fragments $N_1$ and $N_3$ are executed after which $ot_1, ot_2$ synchronize on $t_{e_1}$ labeled with $act_1$ again. Note, that this exactly replays the behavior recorded in the event log $L_{\pat}$. For the "option to complete" property, object type $ot_1$ may only exit the fragment $trans(AN_{\pat})$ by synchronizing with $ot_2$, thereby marking the exiting places discovered for the object interactions with loops pattern.

\subsubsection*{Similar Activity Extension of OCPD Approach:} \janik{Given an event log $L = (E, \OR) $ containing the object interactions with loops pattern. Identify all events $e \in E$ that constitute the object interactions with loops pattern and identify the corresponding two object types $ot_1$ and $ot_2$. The similar activity extension of the OCPD approach works by first applying the original OCPD approach $ocpd_{base}$ and then transforming the resulting object-centric WF-net as follows. }

\noindent\textbf{Step 1.} Apply the original OCPD approach $ocpd_{base} = AN$ with $AN = \\(ON, \MI, \MF$), $ON = (N, pt, F_{var})$, and $N = (P, T, F, l)$.

\noindent\textbf{Step 2.} Apply a post-transformation $trans \in \UAN \rightarrow \UAN$ that transforms the labeled Petri net $N$ underlying the accepting object-centric Petri net $AN$.

\noindent\textbf{Step 2a.} The transformation starts by adding a place $p_1$ of type $ot_1$ 
and by adding two silent transitions, i.e., $T' = T \cup \{ t_1, t_2 \}$ for $t_1, t_2 \not\in T$, $l' = l \oplus (t_1, \tau) \oplus (t_2, \tau)$.

\noindent\textbf{Step 2b.} The flow relation is transformed by first removing all arcs that connect transition $t_{e_1}$ labeled with activity $act_1$ to places of its post-set, 
i.e., $F^b = F \setminus \{ (t_{e_1}, p) | \exists_{t_{e_1} \in T}\, l(t_{e_1}) = act_1 \wedge p \in  t_{e_1}\!\!\bullet \wedge pt(p) = ot_1 \}$.

\noindent\textbf{Step 2c.} The transformation adds arcs to connect transition $t_{e_1}$ with the new place, the new place with a new silent transition and the new silent transitions with all places in the original post-set of the transition $t_{e_1}$ for object type $ot_1$, i.e.,  $ F^c = F^b \cup \{ (t_{e_1}, p_1), (p_1, t_1) \} \cup \{ (t_1, p) | \exists_{t_{e_1} \in T}\, l(t_{e_1}) = act_1 \wedge p \in  t_{e_1}\!\! \bullet \wedge pt(p) = ot_1 \}$.

\noindent\textbf{Step 2d.} The transformation adds arcs to connect the newly added silent transition $t_1$ with the place in the original post-set of transition $t_{e_1}$ of object type $ot_2$, i.e., $F^d = F^c \cup \{ (t_1, p) | \exists_{t_{e_1} \in T}\, l(t_{e_1}) = act_1 \wedge p \in t_{e_1} \!\!\bullet \wedge pt(p) = ot_2 \}$.

\noindent\textbf{Step 2e.} Next, the transformation connects the place in the pre-set of transition $t_{e_2}$ labeled with $act_2$ with the newly added silent transition $t_1$ and $t_2$ and $t_2$ with the place in the pre-set of transition $t_{e_1}$, i.e., $F^e = F^d \cup \{ (p, t_1), (p, t_2), (t_2, p') |\\ \exists_{t_{e_2} \in T} \,l(t_{e_2}) = act_2 \wedge p \in \bullet t_{e_2} \wedge pt(p) = ot_2 \wedge p' \in t_{e_2} \!\!\bullet \wedge pt(p') = ot_2 \}$.

\noindent\textbf{Step 2f.} Then, the place in the pre-set of transition $t_{e_2}$ gets an incoming arc from transition $t_{e_1}$ such that it can fire its transition again after executing the silent transition, i.e., $F^f = F^e \cup \{ (t_{e_1}, p) | \exists_{t_{e_1}, t_{e_2} \in T}\, l(t_{e_1}) = act_1 \wedge l(t_{e_2}) = act_2 \wedge \\p \in \bullet t_{e_2} \wedge pt(p) = ot_2 \}$.

\noindent\textbf{Step 2g.} Finally, variable arcs are transformed such that paths of variable arcs in the original WF-net can still be traversed by variable arcs in the transformed WF-net, i.e. $F'_{var} = F_{var} \cap F' \cup \{ (n_1', n_2') \in F\setminus F' | \exists_{(n_1, n_2) \in F_{var}^*} n_1 = n_1' \wedge n_2 = n_2' \}$.

\noindent\textbf{Step 3.} Return the transformed object-centric Petri net $trans(AN) = AN'$ with $AN' = (ON', \MI, \MF)$, $ON' = (N', pt', F'_{var})$, $N' = (P', T', F', l').$

\janik{Transforming the object-centric Petri net changes the workflow of object type $ot_2$ from a sequential order to a loop such that it can "participate" in each loop cycle of object type $ot_1$ and both exit the loop by synchronizing on the silent transition $t_1$.} \technical{Similar to the different activity extension, we prove the similar activity extension to discover sound process models for event logs containing the object interactions with loops pattern.}

\begin{theorem}
\label{thm:sasound}
\technical{If the Inductive miner is used for process discovery on flattened event logs, OCPD Approach extended with similar activity is a $PT$-sound OCPD technique for $PT = \{ \pat \}$.}
\end{theorem}

\begin{proof}
\technical{
The proof is analogous to the proof of \autoref{thm:dasound} such that we only need to prove that the place-bordered fragment $AN_{\pat}$ is a sound object-centric WF-net after the post-transformation $trans$ is applied
. As aforementioned, the place-bordered fragment $trans(AN_{\pat})$ is sound (cf. \autoref{fig:sa}). $\square$}
\end{proof}

\technical{All in all, it is proven through \autoref{thm:dasound} and \autoref{thm:sasound} that the two extensions overcome the object interactions with loops limitation of the OCPD approach $ocpd_{base}$.}

In general, we cannot prefer one approach over the other, as they treat the ``problematic'' activity (cf. ``receive request'' in \autoref{tab:motivating-events-table} and \autoref{fig:loop}) in the DO-part of the loop that first records no object interaction, but records an object interaction in the second execution of the DO-part, fundamentally different.

The different activity approach interprets the two events recorded for the DO-part of the loop as referring to two different activities in spite of the same activity label. Cases that support the different activity approach are based on information systems that record a too coarse-grained semantic granularity for activity labels in its events such that the label equivalence becomes too imprecise. Moving from label equivalence to attribute equivalence \cite{DBLP:journals/jwsr/Rinderle-MaRJ11} tackles the imprecise distinction between two events as referring to two different real-world activities solely based on the activity label for event logs from those information systems. By taking the object interactions into account, the two events recorded for the DO-part refer to two different real-world activities. In contrast, the similar activity approach maintains the interpretation that the ``problematic'' activity is the same despite having no object interaction recorded between the two object types in question. Cases that support the similar activity approach are data quality issues in the form of a missing object or additional domain knowledge that leads to the decision that the activity label recorded in the events refers to the same executed activity in the real world. 

For example, the event log depicted in \autoref{tab:motivating-events-table} is recorded by an information system that uses a too coarse-grained semantic granularity for its activity labels (cf. collaboration model in \autoref{fig:interaction_da}). The first "receive request" recorded in the event with id 6b0b9 refers to an activity in which the coordinator receives a new request from service provider. After delegating the request to Customer, Customer sends the request back to the coordinator, i.e., an old request is received from a participant that is supposed to handle the request. Hence, the event with id 9c7f8 refers to a real-world activity that is different to the previously referred real-world activity. 

Depending on the interpretation of the mismatch between activity labels and object interactions in practice, either the different activity or the similar activity extension is beneficial. Both extensions are prototypically implemented in \url{https://gitlab.com/janikbenzin/ocpd/} by extending the original OCPD approach implemented in the Python library PM4PY\footnote{\url{https://pm4py.fit.fraunhofer.de/}}

\subsection{Spurios Interactions}
\label{ssec:spurious}

The spurious interaction limitation can be overcome by first relabeling one of the two different events $e_1, e_2$ that satisfy the pattern $\varphi_{si}$ (cf. \autoref{ssec:limreplay}) to a new activity label $act_{new} \in \UACT \setminus \;\mathbb{A}$ for $\mathbb{A}$ the set of activities appearing in the event log before applying $ocpd_{base}$. After discovery of an accepting object-centric Petri net by $ocpd_{base}$ the transition labeled with $act_{new}$ is relabeled back to the original activity label $act$ of the two events $e_1, e_2$ that satisfied the pattern. This relabeling approach to overcome the spurious interactions limitations is also prototypically implemented by extending the original OCPD approach (cf. \autoref{ssec:patterns}).



\section{Related Work}
\label{sec:rel}

For a comprehensive overview for related work on the OCPD approach in terms of classical process discovery and object-centric process discovery, we refer to \cite{van_der_aalst_discovering_2020}. \cite{nesterov_discovering_2023} proposes a compositional object-centric process discovery technique for multi-agent systems that takes the system architecture in terms of interaction patterns into account. Various common synchronous and asynchronous interaction patterns are defined using labeled Petri nets. Given an interface pattern, the technique searches for a series of structural Petri net refinement transformations that are soundness-preserving \cite{bernardinello_property-preserving_2022} to map parts of the given interface pattern with parts of a process model discovered for each agent individually. If a mapping can be found, then the overall discovered process model is guaranteed to be sound. Due to the additional input of an interface pattern and the limited set of transformations, the technique in \cite{nesterov_discovering_2023} cannot discover process models for the settings our extensions can handle. 

\cite{lomazova_soundness_2021} study properties of object-centric Petri nets without taking the discovery technique into account. \cite{lomazova_soundness_2021} propose a variant of our sound object-centric WF-net definition that focuses on a single object $o$ of a certain object type and ignores the behavior of other objects that are required to complete $o$'s workflow. Hence, our notion of a sound object-centric WF-net is stricter. \cite{van_der_werf_correctness_2022} generalize object-centric Petri nets to Petri nets with Identifiers and prove decidability and verification properties of the generalized class of Petri nets.

\section{Conclusion and Limitations}
\label{sec:concl}

Analogous to classical process discovery, OCPD takes an object-centric event log as input and discovers a process model that represents the real-world business process in terms of the control-flow recorded in the event log. \final{By conceptualizing similarly behaving entities in a collaborative system as object types, we discover a process model of the collaborative system instead of a business process.} For the only existing OCPD approach that discovers object-centric Petri nets, we identify the two limitations object interactions with loops and spurious interactions \final{for discovering process models of collaborative systems.} \final{The first limitation is proven to result in an unsound process model, while for the second limitation it is demonstrated that the resulting process model restricts the behavior of the process model without support in the event log.} 
Both limitations are formalized by means of a pattern contained in the event log. We propose three extensions for the OCPD approach to overcome the two limitations. For the two extensions that target the object interactions with loops limitation, we design the extensions such that it results in sound process models given event logs containing the pattern. 

\final{Nevertheless, our set of interaction patterns that represent a limitation of the OCPD approach is limited to two, although there exist more interaction patterns. Moreover, we do not provide a proven statement on discovery of sound process models for all patterns an event log can contain. Hence, we only demonstrate desired properties for the process model given two limitations.}

\subsection*{Acknowledgments}
This work has been supported by Deutsche Forschungsgemeinschaft (DFG), GRK 2201 and by the Austrian Research Promotion Agency (FFG) via the Austrian Competence Center for Digital Production (CDP) under the contract number 881843.
\bibliographystyle{splncs04}
\bibliography{main}

\begin{thebibliography}{10}
\providecommand{\url}[1]{\texttt{#1}}
\providecommand{\urlprefix}{URL }
\providecommand{\doi}[1]{https://doi.org/#1}

\bibitem{van_der_aalst_soundness_2011}
van~der Aalst, W.M.P., van Hee, K.M., ter Hofstede, A.H.M., Sidorova~et al.,
  N.: Soundness of workflow nets: classification, decidability, and analysis.
  Form. Asp. Comput.  \textbf{23}(3),  333--363 (2011)

\bibitem{van_der_aalst_process_2010}
van~der Aalst, W.M.P., Rubin, V., Verbeek, H.M.W., van Dongen~et al., B.F.:
  Process mining: a two-step approach to balance between underfitting and
  overfitting. SoSyM  \textbf{9}(1),  87--111 (2010)

\bibitem{van_der_aalst_process_2016}
van~der Aalst, W.M.P.: Process {Mining}. Springer (2016)

\bibitem{van_der_aalst_object-centric_2019}
van~der Aalst, W.M.P.: Object-{Centric} {Process} {Mining}: {Dealing} with
  {Divergence} and {Convergence} in {Event} {Data}. In: Software {Engineering}
  and {Formal} {Methods}. pp. 3--25. Springer (2019)

\bibitem{van_der_aalst_discovering_2020}
van~der Aalst, W.M.P., Berti, A.: Discovering object-centric {Petri} nets.
  Fundam Inform  \textbf{175}(1-4),  1--40 (2020)

\bibitem{van_der_aalst_object-centric_2017}
van~der Aalst, W.M.P., Li, G., Montali, M.: Object-{Centric} {Behavioral}
  {Constraints} (Mar 2017), \url{http://arxiv.org/abs/1703.05740},
  arXiv:1703.05740 [cs]

\bibitem{artale_enriching_2019}
Artale, A., Calvanese, D., Montali, M., van~der Aalst, W.M.: Enriching {Data}
  {Models} with {Behavioral} {Constraints}. Ontology Makes Sense  \textbf{316},
   257--277 (2019)

\bibitem{artale_object-centric_2017}
Artale, A., Montali, M., Tritini, S., van~der Aalst, W.M.: Object-centric
  behavioral constraints: {Integrating} data and declarative process modelling.
  In: Proceedings of the 30th {International} {Workshop} on {Description}
  {Logics} ({DL}). vol.~1879. CEUR-WS.org (2017)

\bibitem{bernardinello_property-preserving_2022}
Bernardinello, L., Lomazova, I., Nesterov, R., Pomello, L.:
  Property-{Preserving} {Transformations} of {Elementary} {Net} {Systems}
  {Based} on {Morphisms}. In: Koutny, M., Kordon, F., Moldt, D. (eds.)
  Transactions on {Petri} {Nets} and {Other} {Models} of {Concurrency} {XVI},
  pp. 1--23. LNCS, Springer, Berlin, Heidelberg (2022)

\bibitem{van_eck_guided_2017}
van Eck, M.L., Sidorova, N., van~der Aalst, W.M.P.: Guided {Interaction}
  {Exploration} in {Artifact}-centric {Process} {Models}. In: 2017 {IEEE} CBI.
  vol.~01, pp. 109--118 (Jul 2017)

\bibitem{van_eck_multi-instance_2019}
van Eck, M.L., Sidorova, N., van~der Aalst, W.M.P.: Multi-instance {Mining}:
  {Discovering} {Synchronisation} in {Artifact}-{Centric} {Processes}. In:
  Daniel, F., Sheng, Q.Z., Motahari, H. (eds.) BPM {Workshops}. pp. 18--30.
  LNBIP, Springer International Publishing, Cham (2019)

\bibitem{fdhila_verifying_2022}
Fdhila, W., Knuplesch, D., Rinderle-Ma, S., Reichert, M.: Verifying compliance
  in process choreographies: {Foundations}, algorithms, and implementation.
  Information Systems p. 101983 (Jan 2022)

\bibitem{jacobson_credit_2005}
Jacobson, T., Lindé, J., Roszbach, K.: Credit risk versus capital requirements
  under {Basel} {II}: are {SME} loans and retail credit really different?
  Journal of Financial Services Research  \textbf{28},  43--75 (2005),
  publisher: Springer

\bibitem{jung_business_2004}
Jung, J.y., Hur, W., Kang, S.H., Kim, H.: Business process choreography for
  {B2B} collaboration. IEEE Internet Computing  \textbf{8}(1),  37--45 (Jan
  2004)

\bibitem{leemans_discovering_2013}
Leemans, S.J.J., Fahland, D., van~der Aalst, W.M.P.: Discovering
  {Block}-{Structured} {Process} {Models} from {Event} {Logs} - {A}
  {Constructive} {Approach}. In: Application and {Theory} of {Petri} {Nets} and
  {Concurrency}. pp. 311--329 (2013)

\bibitem{li_automatic_2017}
Li, G., de~Carvalho, R.M., van~der Aalst, W.M.P.: Automatic {Discovery} of
  {Object}-{Centric} {Behavioral} {Constraint} {Models}. In: Abramowicz, W.
  (ed.) Business {Information} {Systems}. pp. 43--58. Springer International
  Publishing, Cham (2017)

\bibitem{lomazova_soundness_2021}
Lomazova, I.A., Mitsyuk, A.A., Rivkin, A.: Soundness in {Object}-centric
  {Workflow} {Petri} {Nets} (Dec 2021), \url{http://arxiv.org/abs/2112.14994},
  arXiv:2112.14994 [cs]

\bibitem{de_medeiros_workflow_2003}
de~Medeiros, A.K.A., van~der Aalst, W.M.P., Weijters, A.J.M.M.: Workflow
  {Mining}: {Current} {Status} and {Future} {Directions}. In: Meersman, R.,
  Tari, Z., Schmidt, D.C. (eds.) On {The} {Move} to {Meaningful} {Internet}
  {Systems} 2003: {CoopIS}, {DOA}, and {ODBASE}. pp. 389--406. LNCS, Springer,
  Berlin, Heidelberg (2003)

\bibitem{nesterov_discovering_2023}
Nesterov, R., Bernardinello, L., Lomazova, I., Pomello, L.: Discovering
  architecture-aware and sound process models of multi-agent systems: a
  compositional approach. SoSyM (1),  351--375 (2023)

\bibitem{DBLP:journals/jwsr/Rinderle-MaRJ11}
Rinderle{-}Ma, S., Reichert, M., Jurisch, M.: On utilizing web service
  equivalence for supporting the composition life cycle. Int. J. Web Serv. Res.
   \textbf{8}(1),  41--67 (2011)

\bibitem{sundaramurthy_control_2003}
Sundaramurthy, C., Lewis, M.: Control and {Collaboration}: {Paradoxes} of
  {Governance}. Acad Manage Rev  \textbf{28},  397--415 (Jul 2003)

\bibitem{weijters_flexible_2011}
Weijters, A., Ribeiro, J.: Flexible {Heuristics} {Miner} ({FHM}). In: 2011
  {IEEE} {CIDM}. pp. 310--317 (Apr 2011)

\bibitem{van_der_werf_correctness_2022}
van~der Werf, J.M.E.M., Rivkin, A., Montali, M., Polyvyanyy, A.: Correctness
  {Notions} for {Petri} {Nets} with {Identifiers} (Dec 2022),
  \url{http://arxiv.org/abs/2212.07363}, arXiv:2212.07363 [cs]

\bibitem{van_zelst_discovering_2018}
van Zelst, S.J., van Dongen, B.F., van der Aalst, W.M.P., Verbeek, H.M.W.:
  Discovering workflow nets using integer linear programming. Computing
  \textbf{100}(5),  529--556 (2018)

\bibitem{van_zelst_translating_2020}
van Zelst, S.J., Leemans, S.J.J.: Translating {Workflow} {Nets} to {Process}
  {Trees}: {An} {Algorithmic} {Approach}. Algorithms  \textbf{13}(11) (2020)

\end{thebibliography}
%



\end{document}